\documentclass{article}

\usepackage{spconf,amsmath,epsfig}
\usepackage{amssymb}
\usepackage{amsthm}
\usepackage{algorithm}
\usepackage{algpseudocode}
\usepackage{tikz}
\usepackage{tabu}
\usepackage{graphicx}
\usepackage{subfig}
\usepackage{epstopdf}
\usepackage{epsfig}
\usetikzlibrary{tikzmark,calc}
\usepackage{caption}
\usepackage{float}
\usepackage{multirow}
\numberwithin{equation}{section} 
\newtheorem{theorem}{Theorem}[section]                   

\usepackage{microtype}

\def\R{{\mathbb{R}}}

\DeclareMathOperator*{\argmin}{arg\,min}
\DeclareMathOperator*{\argmax}{arg\,max}

\title{Stable recovery of sparse vectors from \\ random sinusoidal feature maps}
\name{Mohammadreza Soltani and Chinmay Hegde
\thanks{This work is supported in part by the National Science Foundation under the grants CCF-1566281 and IIP-1632116.}}
\address{Department of Electrical and Computer Engineering, Iowa State University}

\begin{document}
\ninept

\maketitle
\begin{abstract}
Random sinusoidal features are a popular approach for speeding up kernel-based inference in large datasets. Prior to the inference stage, the approach suggests performing dimensionality reduction by first multiplying each data vector by a random Gaussian matrix, and then computing an element-wise sinusoid. Theoretical analysis shows that collecting a sufficient number of such features can be reliably used for subsequent inference in kernel classification and regression.

In this work, we demonstrate that with a mild increase in the dimension of the embedding, it is also possible to reconstruct the data vector from such random sinusoidal features, provided that the underlying data is sparse enough. In particular, we propose a numerically stable algorithm for reconstructing the data vector given the nonlinear features, and analyze its sample complexity. Our algorithm can be extended to other types of structured inverse problems, such as demixing a pair of sparse (but incoherent) vectors. We support the efficacy of our approach via numerical experiments.
\end{abstract}

\section{Introduction}
\label{sec:intro}

\subsection{Setup}

Several popular techniques in statistical regression and classification rely on the \emph{kernel trick}, which enables approximating complicated nonlinear functions given a sufficiently large number of training samples. However, in large-scale problems where the number as well as dimension of the training data points is high, the computational costs of kernel-based techniques can be severe. 

One approach to alleviate these computational costs is via the use of \emph{random sinusoidal feature maps}, pioneered by~\cite{rahimi2007random}. The mechanism is conceptually simple. Suppose that the kernel of choice is the Gaussian (RBF) kernel. Prior to the inference stage, the approach suggests performing a \emph{dimensionality reduction} of a given data vector by first multiplying by a random Gaussian matrix, followed by the application of a nonlinearity (specifically, a sinusoid or complex exponential). 
Mathematically, one can represent the above process as follows. Let $\mathcal{X} \subseteq \R^{n}$ be any set in the data space, and consider a data vector $x \in \mathcal{X}$.  Construct a matrix $B \in \R^{q \times n}$ whose entries are independently from a normal distribution. Then, modulo scaling factors, the random sinusoidal feature map, $y = \mathcal{A}(x)$, can be represented as the composition of the linear map $A\in\mathbb{R}^{m\times n}$ with a (complex) sinusoidal nonlinearity (see section~\ref{sec:alg} for more details):
\begin{align}
\label{eq:sine_obs}
z_j &= (Ax)_j,~~~y_j= \exp(i \, z_j),~~~~j=1\dots m\ .
\end{align}

One can alternately replace the complex exponential by a real-valued sinusoidal function with similar consequences. The authors of~\cite{rahimi2007random} show that collecting a sufficient number of such features leads to statistically robust performance of kernel-based learning methods. Subsequent works~\cite{additivekernels,hashkernels,fastfood,nonlincompanalysis} have progressively extended this approach to more general types of kernels, faster methods of dimensionality reduction, as well as sparsity assumptions on the training data~\cite{aswinkernelsparse}.

Moving beyond inference, a natural question arises whether the original data vector $x$ can at all be \emph{reconstructed} from the nonlinear features $y$. Geometrically, we would like to understand if (and how) the random feature map can be efficiently inverted, given the observations $y$ and knowledge of the dimensionality reduction operator $B$. This question is of both theoretical and practical interest; if the embedding is shown to be \emph{reversible}, then it could have implications for the privacy-preserving properties of random feature maps.

Observe that \eqref{eq:sine_obs} is analogous to the well-studied statistical notion of a generalized linear model (GLM) with a sinusoidal \emph{link} (or transfer) function. Here, the data vector $x \in \R^n$ represents the model parameters. Reconstruction of high dimensional vectors in the GLM setting has been an intense focus of study in the statistical learning literature in recent years; applications of such recovery methods span scalable machine learning~\cite{kakade2011}, 1-bit compressive sensing~\cite{boufounos20081,jacques1bit}, and imaging~\cite{candes2015phase}. In signal processing applications, observation models with periodic link functions have been proposed for scalar quantization~\cite{boufounosquant} as well as computational imaging~\cite{modulocamera}. 

\subsection{Contributions}

We consider a slightly more general version of the model~\eqref{eq:sine_obs} where the features $y$ can be corrupted by noise. Moreover, we assume that the data vector $x$ is $s$-sparse, i.e., it contains no more than $s$ non-zeros. For this model, we propose an efficient, two-stage algorithm for reconstructing signals from a small number of random sinusoidal features (in particular, far smaller than the dimension of the data). To the best of our knowledge, our method is the first to propose an algorithm that is specialized to reconstructing sparse vectors from random sinusoidal feature maps.

To supplement our algorithm, we provide a theoretical sample complexity analysis. Our theory demonstrates that stable recovery from nonlinear sinusoidal observations can be done with \emph{essentially} the same number of samples as standard sparse recovery with linear observations. The only additional cost is a logarithmic factor overhead in the dimension of the embedding that depends on $n$ and the Euclidean norm of the original data vector. We also support the theory and algorithms via a range of numerical experiments.

\subsection{Techniques}

There are two main challenges that we will need to overcome. The first challenge is due to the ill-posed nature of the problem; in the high-dimensional regime, the total number of samples can be far less than the native dimension of the data. The second (and more compelling) challenge arises due to the fact that the sinusoidal transfer function is highly non-invertible. The typical way to deal with this is to assume some upper bound on the magnitude of the entries of the linear projection $(Ax)$. However, we observe that unless this upper bound is relatively small, we can only estimate $z_j$ up to some unknown integer multiple of $2\pi$. If the projected dimension of the linear embedding, $q$, is large, then the number of possible combinations becomes exponential in $q$. 

We alleviate these challenges using a simple idea. The key is to \emph{decouple} the problem of inverting the sinusoidal nonlinearity from the sparse recovery step by replacing the linear map $A$ with a different (carefully designed) map that still possesses all the properties required for reliable kernel-based inference. Particularly, our new linear map $A$ is the product of two matrices $D \cdot B$; the way to construct $D$ and $B$ is described in detail below. This technique has been inspired by some recent approaches for the related problem of \emph{phase retrieval}~\cite{iwen2015robust,bahmani2015efficient}. Using this decoupling technique, we can separately solve the inversion of the nonlinearity and the actual recovery of the sparse high-dimensional data in two stages. 

For the first stage, we \emph{individually} estimate the linear projection $z = Bx$ using classical signal processing techniques for line spectral estimation. In particular, we leverage the method of {matched filtering from random samples}~\cite{eftekhari2013matched}. In the absence of noise, any spectral estimation technique (such as MUSIC, root-MUSIC, and ESPRIT) can also be used; however, our Monte Carlo simulations show below that the randomized method of~\cite{eftekhari2013matched} et al. is considerably more robust compared to these more classical methods. For the second stage, we can use any sparse recovery method of our choice. In our experiments, we have used the CoSaMP algorithm proposed in~\cite{cosamp}.

While conceptually simple, the generic nature of our algorithm is advantageous in many ways. From a theoretical perspective, mirroring the approach of~\cite{iwen2015robust}, we can combine existing results for robust line spectral estimation with robust sparse recovery and obtain guarantees on the sample complexity and robustness of our algorithm. Moreover, since the nonlinear inversion and the sparse recovery steps are decoupled, one could conceivably extend this decoupling approach to a variety of other signal models and recovery algorithms (such as structured-sparsity~\cite{modelcs,approxIT}, low-rank matrix models~\cite{rechtfazelparillo}, and demixing a pair of sparse incoherent vectors~\cite{soltani2016fastIEEETSP17}). 

\subsection{Comparison with prior work}

Our approach relies upon ideas from several related problems in machine learning and signal processing. The technique of using random projections in conjunction with sinusoidal nonlinearities has emerged as a popular alternative for kernel-based inference~\cite{rahimi2007random,additivekernels,hashkernels,fastfood,nonlincompanalysis}. However, these works do not explicitly consider recovering the original data from the nonlinear random projections. 

In parallel, there has been considerable progress for various problems where \emph{reconstruction} of data vectors having some form of low-dimensional structure is of chief interest~\cite{beck2013sparsity,plan2014high,Cons2015NIPS,yang2015sparse,bahmani2011greedy}. Particularly, there has been considerable recent interest in reconstructing from nonlinear measurements (features). However, the majority of the methods deal with the case where the nonlinearity is either monotonic~\cite{kakade2011} or satisfies restricted regularity conditions~\cite{bahmani2011greedy,yang2015sparse}.

None of the proposed methods in the literature can be directly applied for the specific problem of stable sparse reconstruction from random sinusoidal features, with one exception: the theoretical framework developed in~\cite{plan2014high} can indeed be adapted to our problem. However, the recovery algorithm only yields an estimate up to a scalar ambiguity, which can be problematic in applications. Moreover, even if scalar ambiguities can be tolerated, we show in our experimental results below that the sample complexity of this method is far too high in practice. In contrast, our proposed method yields stable and accurate results with only a mild factor increase in sample complexity when compared to standard sparse recovery.

\section{Mathematical model and Algorithm}
\label{sec:alg} 

In this section, we present our algorithm for solving problem~\eqref{eq:sine_obs} and provide the theory for sample complexity of the proposed algorithm. We first establish some notation. 
Let the $j^{\mathrm{th}}$ entry of the signal $x\in\mathbb{R}^n$ be denoted as $x_j$. For any $j \in [q]$, $x(j:q:(k-1)q+j)$ denotes the sub-vector in $\mathbb{R}^k$ formed by the entries of $x$ starting at index $j + qr$, where $r = 0, 1, \ldots, k-1$. Similarly, $A((j:q:(k-1)q,l)$ represents the sub-vector constructed by picking the $l^{\mathrm{th}}$ column of any matrix $A$ and selecting the entries of this column with the same procedure as mentioned. 


As the key component in our approach, instead of using a random Gaussian linear projection as proposed in~\cite{rahimi2007random}, we construct a linear operator $A$ that can be \emph{factorized} as $A = D B$, where $D\in\mathbb{R}^{m\times q}$, and $B\in\mathbb{R}^{q\times n}$. We assume that $m$ is a multiple of $q$, and that $D$ is a concatenation of $k$ diagonal matrices of $q \times q$ such that the diagonal entries in the blocks of $D$ are i.i.d. random variables, drawn from a distribution that we specify later. The choice of $B$ depends on the model assumed on the data vector $x$; if the data is $s$-sparse, then $B$ can be any matrix that supports stable sparse recovery (more precisely, $B$ satisfies the null-space property~\cite{foucart2013}). Overall, our low-dimensional feature map can be written as:
\begin{align}\label{RDMF}
y = \exp(iDBx) + e = \exp\left(i
\begin{bmatrix}
D^1 \\
D^2 \\
 \vdots  \\
D^k
\end{bmatrix}
Bx
\right) + e,
\end{align}
where $D^i$ are diagonal matrices and $e\in\mathbb{R}^m$ denotes additive noise such that $e\sim\mathcal{N}(0,\sigma^2I)$. The goal is to stably recover $x$ from the embedding $y$.

By the block diagonal structure of the outer projection $D$, we can reduce the overall reconstruction problem to first obtaining a good enough estimate of each entry of $Bx$. We show below that this can be achieved using line spectral estimation methods. The output of the first stage is used as the input of the second stage, which involves estimating $x$ from a (possibly noisy) estimate of $Bx$. 


\subsection{Line spectral estimation}\label{ToneEst}

Consider the observation model~\eqref{RDMF} and let $z = Bx \in \R^q$. Suppose we know \emph{a priori} that the entries of $z$ belong to some bounded set $\Omega \in R$. Fix $l \in [q]$, and let {$t = D(l:q:(k-1)q+l,l), u = y(l:q:(k-1)q+l), h = e(l:q:(k-1)q+l)$} which are vectors in $\mathbb{R}^k$. We observe that:
\[
u = \exp(i z_l t) + h \, .
\]
In other words, $u$ can be interpreted as a collection of time samples of a (single-tone, potentially off-grid) complex-valued signal with frequency $z_l \in \Omega$, measured at time locations $t$. Therefore, we can independently estimate $z_l$ for $l=1, \ldots, q$ by solving a least-squares problem~\cite{eftekhari2013matched}:
\begin{align}
\label{OptmExp}
\widehat{z_l} = \underset{v \in \Omega}{\argmin}\|u - \exp(i \, vt )\|_2^2 = \underset{v \in \Omega}{\argmax}\left|\langle u,\psi_{v}\rangle\right|,
 \end{align}
for all $l=1,\ldots,q$, where  $\psi_{v}\in\mathbb{R}^k$ denotes a \emph{template vector} given by $\psi_{v} = \exp(j t v)$ for any $v \in \Omega$. In essence, the solution of the least-squares problem can be interpreted as a \emph{matched filter}. Numerically, the optimization problem in~\eqref{OptmExp} can be solved using a grid search over the set $\Omega$, and the resolution of this grid search controls the running time of the  algorithm; for fine enough resolution, the estimation of $z_l$ is more accurate albeit with increased running time. This issue is also discussed in~\cite{eftekhari2013matched} and more sophisticated spectral estimation techniques have appeared in~\cite{eldarxampling,tangcsoffgrid,chioffgrid}. After obtaining all the estimates $\widehat{z_l}$'s, we stack them in a vector $\widehat{z}$. 

\subsection{Sparse recovery}
We now propose to recover the high-dimensional signal $x \in \R^m$ in problem~\eqref{RDMF} using our estimated $\widehat{z} \in \R^q$. According to our model, we also have assumed that the matrix $B$ supports {stable sparse recovery}, and the underlying signal $x$ is $s$-sparse. Hence, we can use any generic sparse recovery algorithm of our choice to obtain an estimate of $x$. In our simulations below, we use the CoSaMP algorithm~\cite{cosamp} due to its speed and ease of parameter tuning. Again, we stress that this stage depends on a model assumption on $x$, and other recovery algorithms (such as structured-sparse recovery~\cite{modelcs}, low-rank matrix recovery~\cite{rechtfazelparillo}, or demixing a pair of sparse (but incoherent) vectors~\cite{soltani2016fastIEEETSP17}) can equally well be used. Our overall algorithm, \emph{Matched Filtering+Sparse recovery} (MF-Sparse) is described in pseudocode form in Algorithm~\ref{alg:RDMF}.
We now provide a theoretical analysis of our proposed algorithm. Our result follows from a concatenation of the results of~\cite{eftekhari2013matched} and~\cite{cosamp}.

\begin{theorem}[Sample complexity of MF-Sparse]
Assume that the nonzero entries of $D$ are i.i.d.\ samples from a uniform distribution $[-T,T]$, and the entries of $B$ are i.i.d.\ samples from $\mathcal{N}(0,1/q)$. Also, assume $\|x\|_2 \leq R$ for some constant $R>0$. Set $m = kq$ where $k=c_1 \log\left(\frac{Rq}{\varepsilon}\frac{1}{\delta}\right)(1+\sigma^2)$ for some $\varepsilon>0$ and $q = c_2 \left(s\log\frac{n}{s}\right)$. Set $\omega = c_3 R$ and $\Omega = [-\omega,\omega]$. Then, MF-Sparse returns an estimate $\widehat{x}$, such that 
$$\|x - \widehat{x}\|_2 \leq C \varepsilon \, ,$$ 
with probability exceeding $1-\delta$. Here, $c_1, c_2, c_3, C$ are constants. 
\end{theorem}

\begin{proof}
Set $\varepsilon'>0$. Based on Corollary 8 of~\cite{eftekhari2013matched}, by setting $|T| = \mathcal{O}(\frac{1}{\varepsilon'})$, we need $k=\mathcal{O}\left((1+\sigma^2)\log\left(\frac{|\Omega|}{\varepsilon'}\frac{1}{\delta}\right)\right)$ to obtain $|\widehat{z_l} -z_l|\leq\varepsilon'$  with probability at least $1-\delta$ for each $l=1, \ldots, q$. Here, $|\Omega|$ denotes the radius of the feasible range for $z$, and appears in the maximization problem~\eqref{OptmExp}. Observe that $|\Omega| = \mathcal{O}\left(\|Bx\|_{\infty}\right)$. If the entries of matrix $B$ are chosen as stipulated in the theorem, then $\|Bx\|_{\infty} \leq  \|Bx\|_{2} \lesssim \mathcal{O}(R)$ which justifies the choice of $|\Omega|$ in the theorem. Thus, $k=\mathcal{O}\left((1+\sigma^2)\log\left(\frac{R}{\varepsilon'}\frac{1}{\delta}\right)\right)$. 
By a union bound, it follows that with probability at least $1- q\delta$, we have $\|\widehat{z} - z\|_{\infty}\leq\varepsilon'$. Now, we can write $\widehat{z} = z + e' = Bx + e'$ where $\|e'\|_{\infty}\leq\varepsilon'$. Since we have used CoSaMP in the sparse recovery stage, if the choice of $q=\mathcal{O}\left(s\log\frac{n}{s}\right)$ enables us to obtain $\|\widehat{x} - x\|_2\leq c\|e'\|_2\leq c\sqrt{q}\|e'\|_{\infty}\leq c\sqrt{q}\varepsilon'$~\cite{cosamp}. (In fact, a more general guarantee can be derived even for the case when $x$ is not exactly $s$-sparse.) Now, let $\varepsilon' = \mathcal{O}\left(\frac{\varepsilon}{\sqrt{q}}\right)$ to obtain the final bound on the estimation error $\|\widehat{x}-x\|_2\leq\mathcal{O}(\varepsilon)$ with $k=\mathcal{O}\left((1+\sigma^2)\log\left(\frac{Rq}{\varepsilon}\frac{1}{\delta}\right)\right)$.
\end{proof}

In model~\eqref{eq:sine_obs}, the features $y_j$ are modeled in terms of complex exponentials. With only a slight modification in the line spectral estimation stage, we can recover $x$ from real-valued random sine features. More precisely, these random features are given by:
\begin{align}
\label{eq:realsine_obs}
y = \sin(DBx) + e
\end{align}
where $e$ denotes the additive noise as defined before. If we follow an analogous approach for estimating $x$, then in lieu of the least squares estimator \eqref{OptmExp}, we have the following estimator:
\begin{align}
\label{OptmSin}
\widehat{z_l} &= \underset{v \in\Omega}{\argmin}\|u - \sin(v t)\|_2^2 \\
&= \underset{v \in \Omega}{\argmax}\left( 2\left|\langle u,\psi_{v}\rangle\right| - \|\psi_{v}\|_2^2\right),
\end{align}
for $l=1,\ldots,q$ and $u$ as defined above. {Also, $\psi_{v} = \sin(tv)$ for any $v \in \Omega$}. We only provide some numerical experiments for this estimator and leave a detailed theoretical analysis for future research.

\begin{algorithm}[t]
\caption{\textsc{MF-Sparse}}
\label{alg:RDMF}
\begin{algorithmic}
\State\textbf{Inputs:} $y$, $D$, $B$, $\Omega$, $s$
\State\textbf{Output:}  $\widehat{x}$
\State \textbf{Stage 1: Tone estimation:}
\For {$l =1:q$}
\State $t \leftarrow D(l:q:(k-1)q+l,l)$
\State $u \leftarrow y(l:q:(k-1)q+l)$
\State $\widehat{z_l} = \argmax_{\omega\in\Omega}|\langle y,\psi_{\omega}\rangle|$
\EndFor
\State $\widehat{z} \leftarrow [\widehat{z_1},\widehat{z_2}\ldots,\widehat{z_q}]^T$
\State \textbf{Stage 2: Sparse recovery}
\State $\widehat{x} \leftarrow \textsc{CoSaMP}(\widehat{z},B,s)$
\end{algorithmic}
\end{algorithm}

\section{Experimental Results}
\label{sec:res}


We compare our algorithm with existing algorithms in various scenarios. In the first experiment, we assume that our random features are computed using a real sine link function. In other words, the features, $\{y_j\}$ is given by~\eqref{eq:realsine_obs} for all $j=1,\ldots,q$.

In Fig.~\ref{fig:sinNonoise}(a), we compare the probability of recovery of MF-Sparse with \textit{Gradient Hard Thresholding} (GHT), a projected-gradient descent type algorithm whose variants have proposed in~\cite{bahmani2011greedy,yuan2014gradient,jain2014iterative}. To recover $x$, GHT tries to minimize a specific loss function (typically, the squared loss) between the observed random feature vector, $y$, and $\sin(DBx)$ by iteratively updating the current estimate of $x$ based on a gradient update rule, and then projecting it into the set of $s$-sparse vectors via hard thresholding. 

The setup for the experiment illustrated in Fig.~\ref{fig:sinNonoise}(a) is as follows. First, we generate a synthetic signal of length $n = {2^{14}}$ with sparsity $s=100$ such that the support is random and the values of the signal in the support are drawn from a standard normal distribution. Then, the $\ell_2$-norm of $x$ is adjusted via a global scaling to coincide with three different values; $\|x\|_2 =1, \|x\|_2 = 15$, and $\|x\|_2 =30$. We generate a matrix $B\in\mathbb{R}^{q\times n}$ where $q =700$ and $n = 2^{14}$ such that the entries of it are i.i.d random variables with distribution $\mathcal{N}(0,\frac{1}{\sqrt{q}})$. All nonzero entries of $D$, $d_{l}^r$ for $l=1,\ldots, q$ and $r = 1, \ldots, k$ are assumed to be standard normal random variables for $k=1,\ldots, 8$. Next, we generate $y\in\mathbb{R}^m$ as $y = \sin(DBx)$ where $m = kq$. (There is no noise considered in this experiment.) By running MF-Sparse and GHT, we obtain the estimate of $x$, denoted by $\widehat{x}$, and calculate the (normalized) estimation error defined as $\frac{\|\widehat{x}-x\|_2}{\|x\|_2}$. We repeat this process for $60$ Monte Carlo trials, and define the empirical probability of {successful} recovery as the fraction of simulations in which the relative error is less than $0.05$. 

As we can see in Fig.~\ref{fig:sinNonoise}(a), MF-Sparse can successfully recover $x$ even with $k=4$ number of blocks with probability close to $1$ when the norm of $x$ is small. In this regime, i.e., $\|x\|_2$ being small, both GHT and MF-Sparse display similar performance. However, GHT shows extremely poor performance when $\|x\|_2$ has moderate or large value. The reason is that when the norm of $x$ is small, the entries of $Bx$ are also small (i.e., the entries of $DBx$ are close to the origin with high probability), and the sinusoidal nonlinearity can be assumed to be almost linear. Consequently, GHT can recover $x$ in this regime. However, this assumption breaks down for larger values of $\|x\|_2$.

\begin{figure}[t]
\begin{center}
\begingroup
\setlength{\tabcolsep}{.1pt} 
\renewcommand{\arraystretch}{.1} 
\begin{tabular}{cc}      
\includegraphics[trim = 8mm 58mm 15mm 60mm, clip, width=0.48\linewidth]{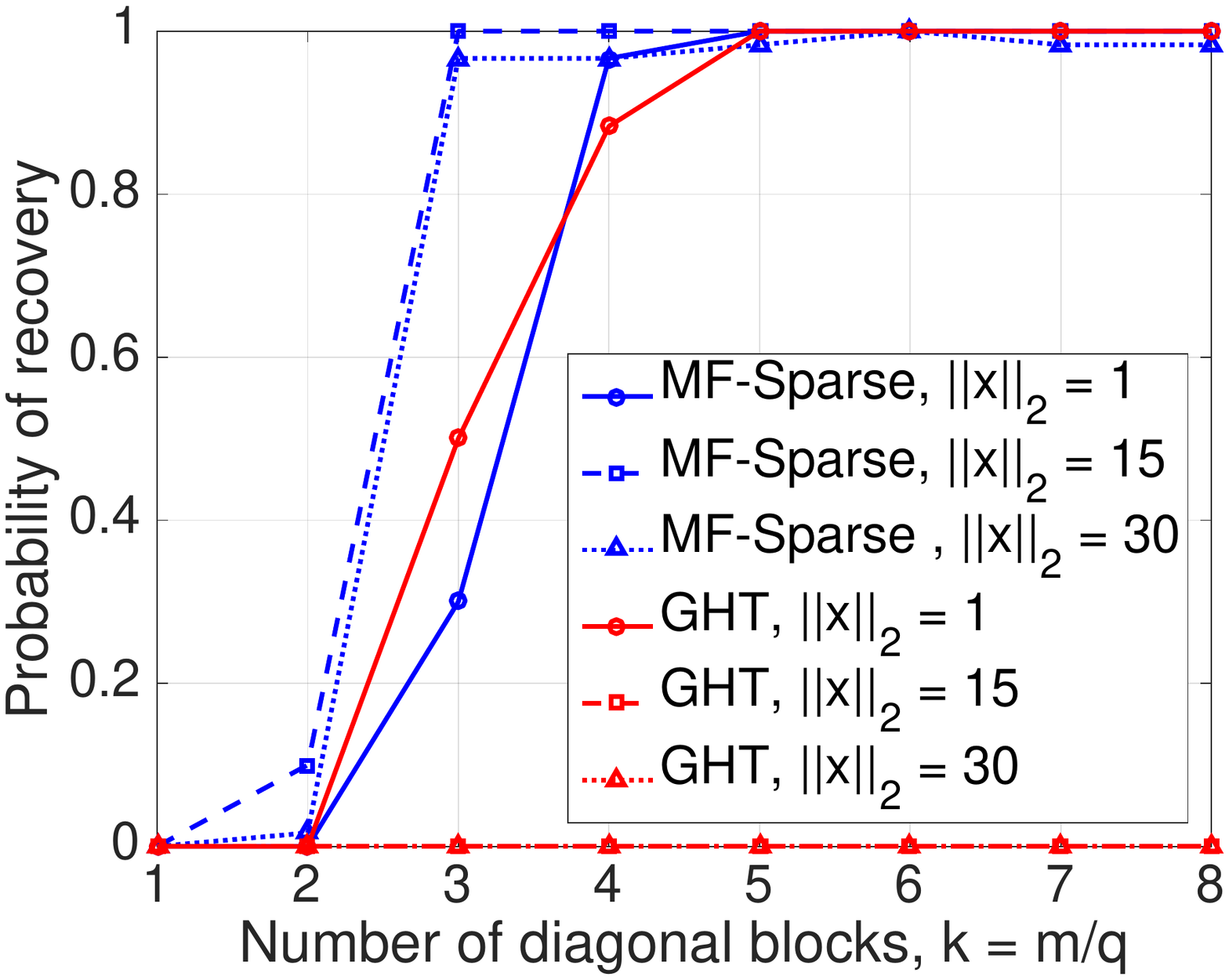}&
\includegraphics[trim = 8mm 58mm 15mm 60mm, clip, width=0.48\linewidth]{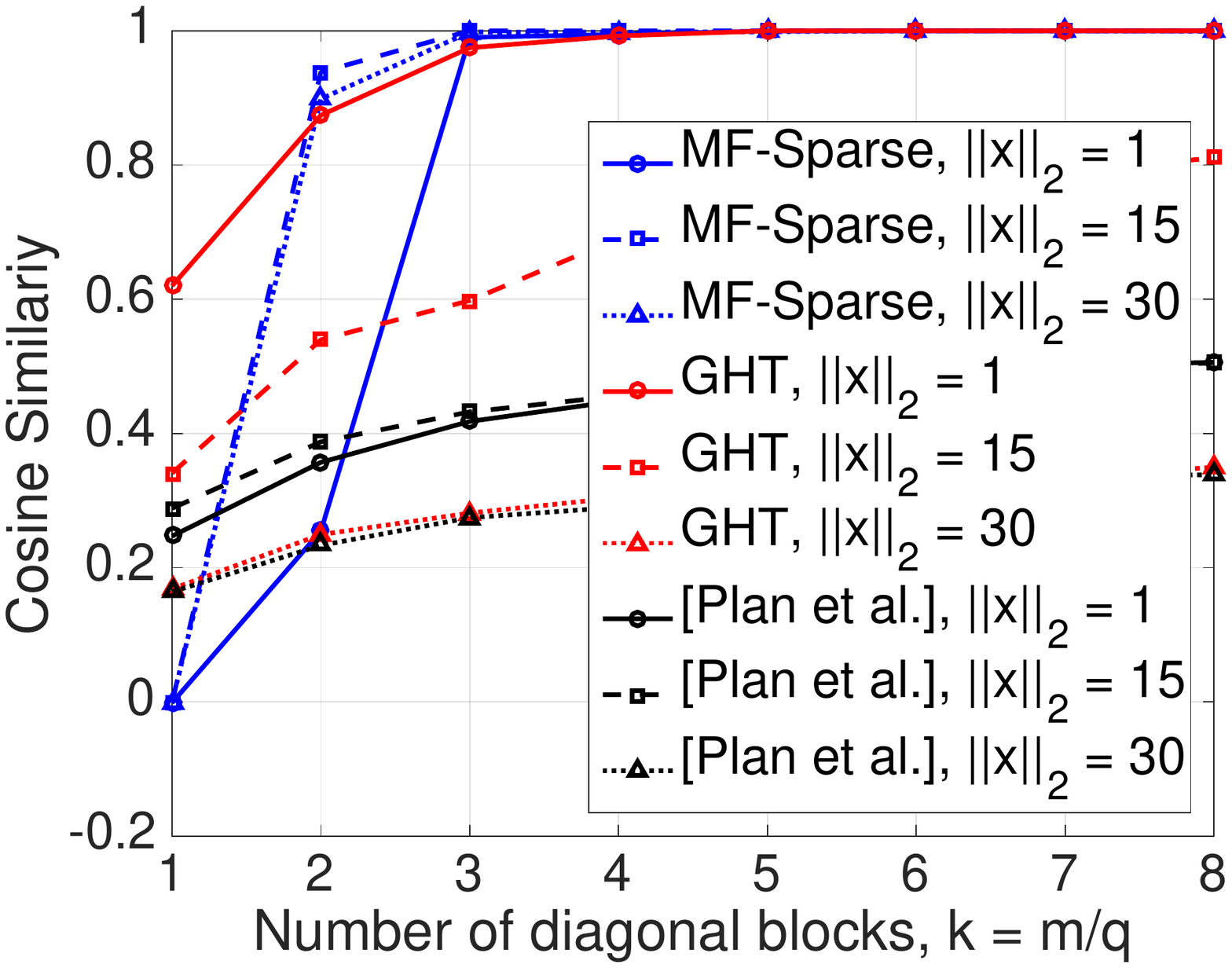}\\
(a) & (b) 
\end{tabular}
\endgroup
\end{center}
\caption{\emph{Comparison of the proposed algorithm with other algorithms. Parameters: $n =2^{14}, q= 700$. (a) Probability of recovery in terms of normalized error. (b) Cosine similarity between recovered and true signal.}}
\label{fig:sinNonoise}
\end{figure}

In Fig.~\ref{fig:sinNonoise}(b), we repeat a similar experiment, but measure performance with a different recovery criterion. In this scenario, we measure the \textit{cosine similarity} of the estimated vector $\widehat{x}$ with $x$, defined as $\frac{\widehat{x}^Tx}{\|\widehat{x}\|_2\|x\|_2}$. In addition to GHT, we also compare the performance of MF-Sparse with the single-step thresholding approach proposed by~\cite{plan2014high}. The approach of~\cite{plan2014high} only recovers the vector modulo an (unknown) scaling factor, and does not need to possess knowledge of the nonlinearity. As illustrated in Fig.~\ref{fig:sinNonoise}(b), the approach of~\cite{plan2014high} has worse performance compared to two other methods. Also, GHT shows good performance only when $\|x\|_2 = 1$, as expected. In contrast, MF-Sparse shows the best performance.

\begin{figure}[t]
\begin{center}
\includegraphics[trim = 5mm 58mm 5mm 72mm, clip, width=0.55\linewidth]{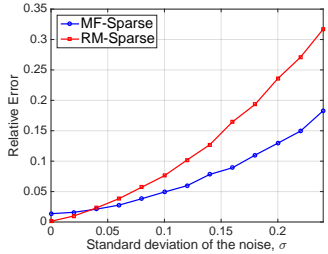}
\end{center}
\caption{\emph{Comparison of matched filtering versus rootMUSIC in the first stage. Parameters:  $n =2^{14}, q= 800$, $k =6$. 
}}
\label{fig:expWnoise}
\end{figure}

\begin{figure}[t]
\begin{center}
\begingroup
\setlength{\tabcolsep}{.1pt} 
\renewcommand{\arraystretch}{.1} 
\begin{tabular}{ccc}      
\includegraphics[trim = 47mm 68mm 15mm 55mm, clip, width=0.32\linewidth]{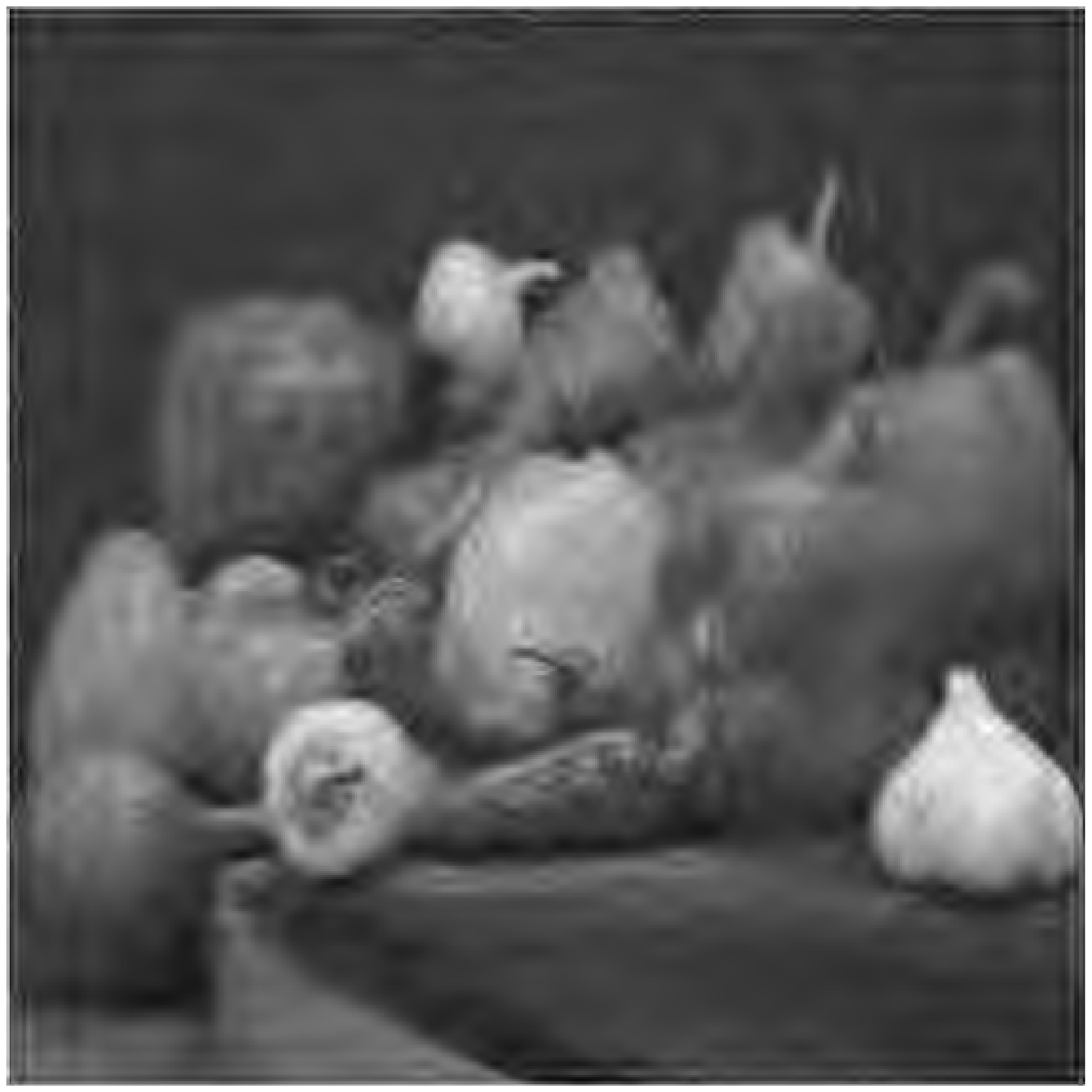}&
\includegraphics[trim = 47mm 68mm 15mm 55mm, clip, width=0.32\linewidth]{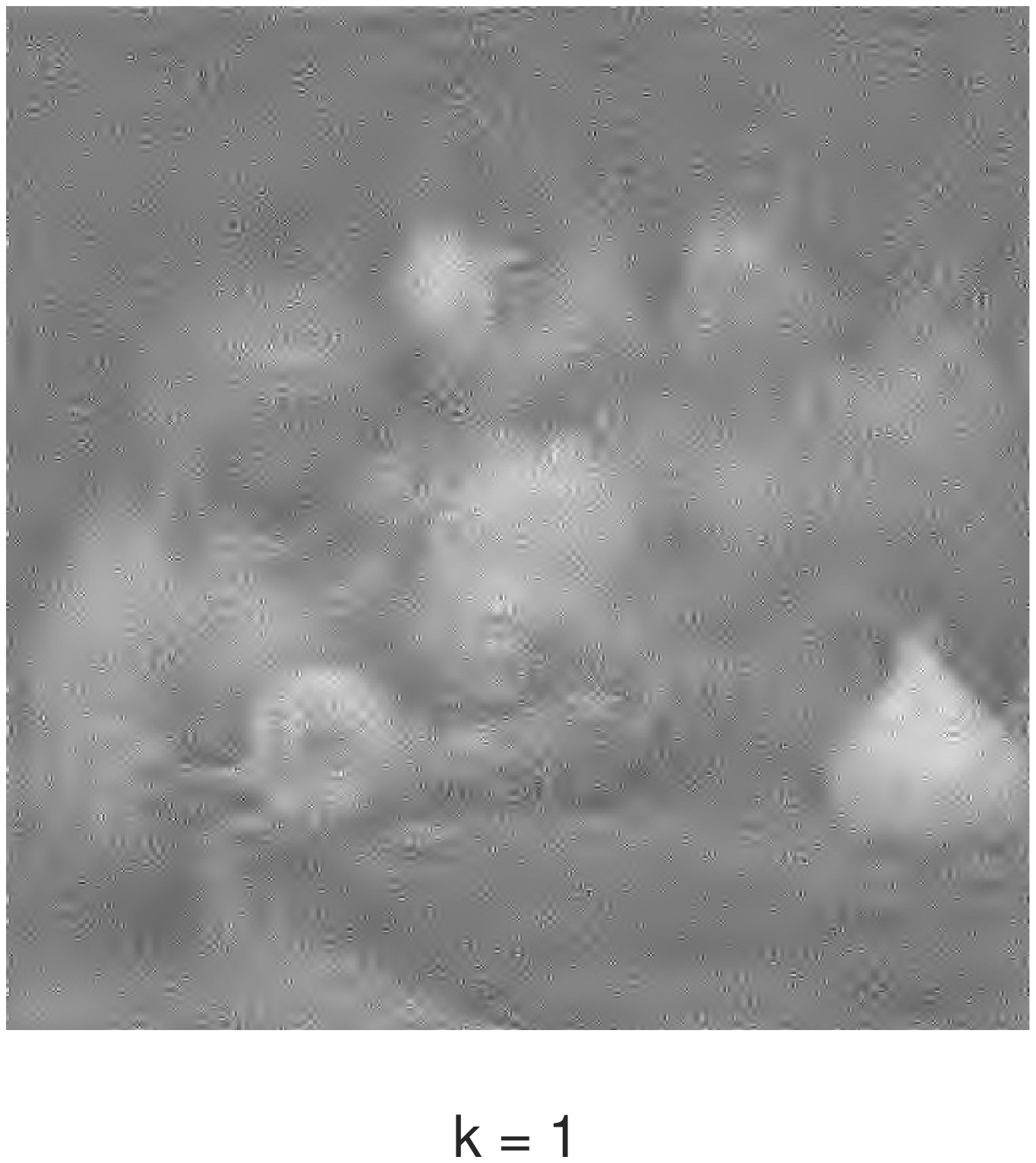} &
\includegraphics[trim = 47mm 68mm 15mm 55mm, clip, width=0.32\linewidth]{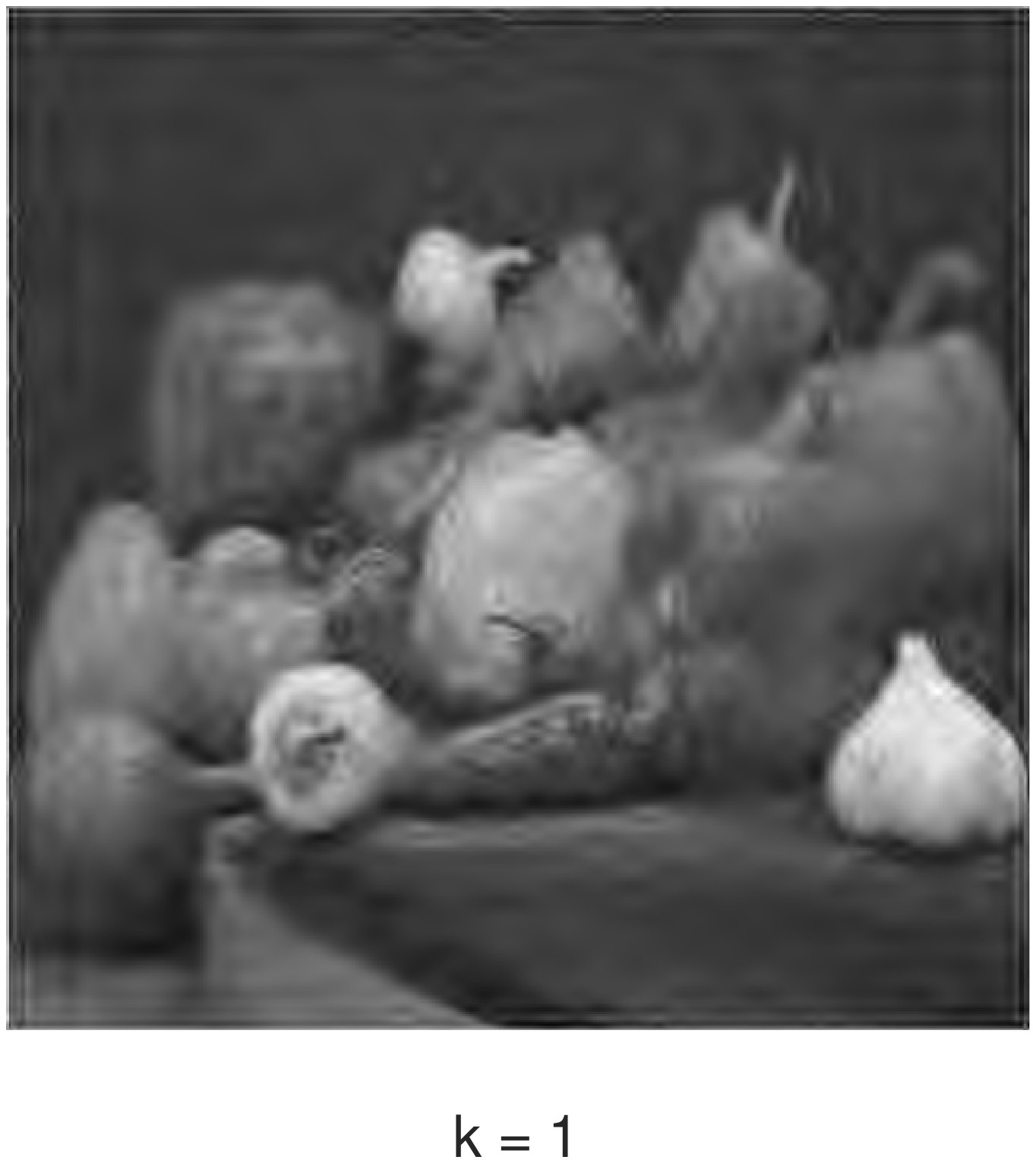} \\
(a) & (b) & (c) 
\end{tabular}
\endgroup
\end{center}
\caption{\emph{Successful reconstruction on a real $2$-D image from random sinusoidal features. (a) Test image. (b) Reconstruction quality with $k=2$ diagonal blocks.
(c) Reconstruction with $k=3$.}}
\label{fig:pepper}
\end{figure}

\begin{figure}[t]
\begin{center}
\begingroup
\setlength{\tabcolsep}{.1pt} 
\renewcommand{\arraystretch}{.1} 
\begin{tabular}{ccc}      
\includegraphics[trim = 47mm 68mm 15mm 55mm, clip, width=0.32\linewidth]{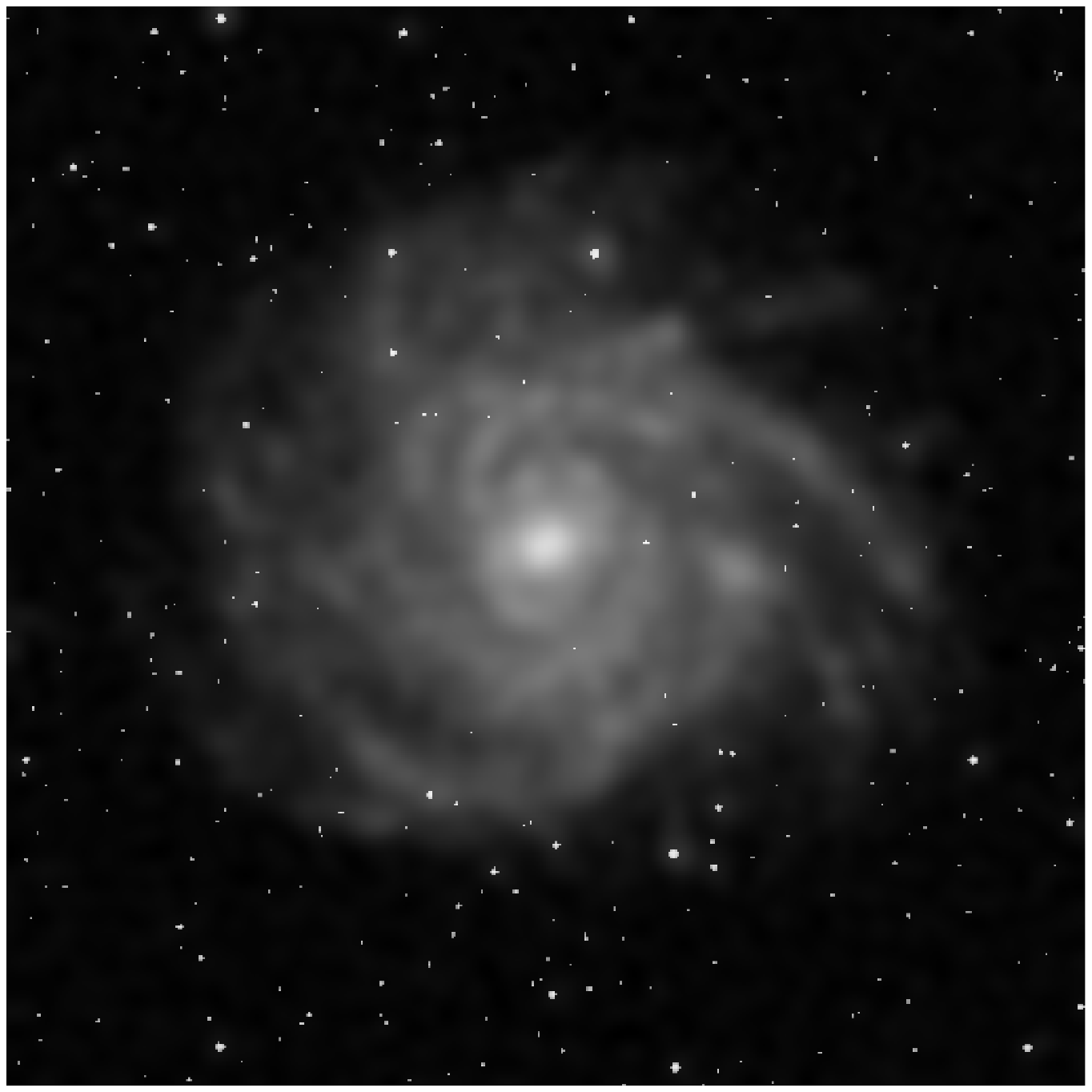}&
\includegraphics[trim = 47mm 68mm 15mm 55mm, clip, width=0.32\linewidth]{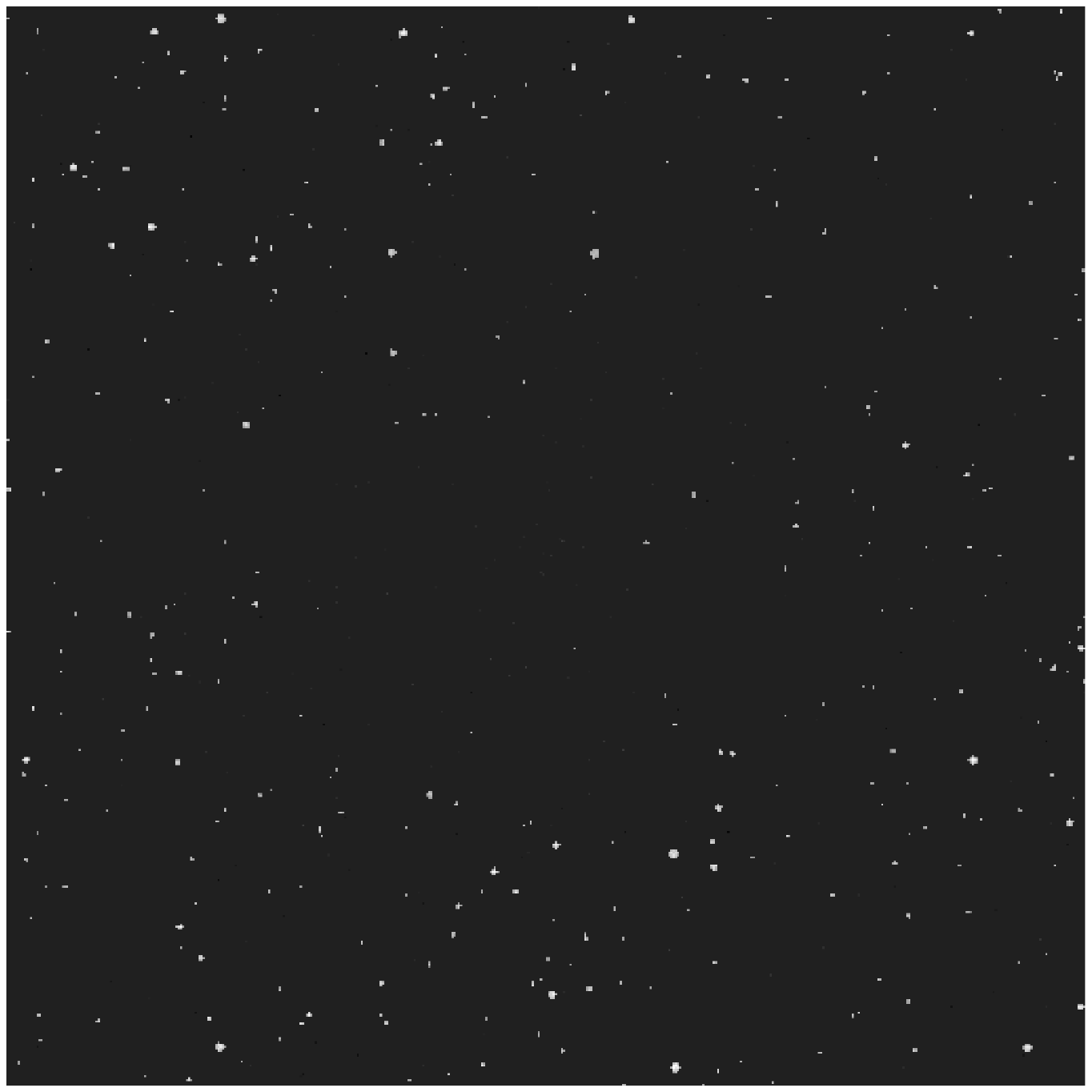} &
\includegraphics[trim = 47mm 68mm 15mm 55mm, clip, width=0.32\linewidth]{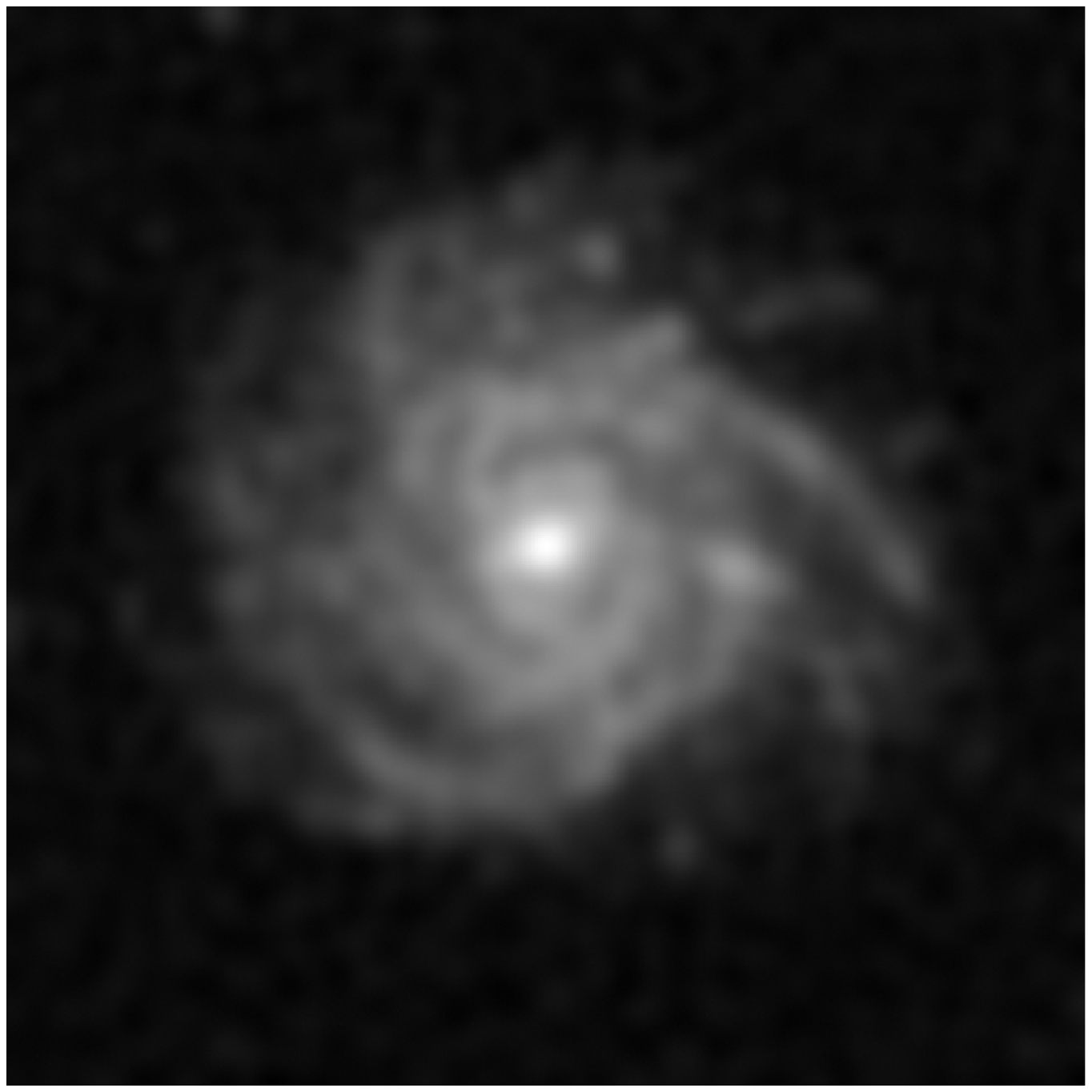} \\
(a) & (b) & (c) 
\end{tabular}
\endgroup
\end{center}
\caption{\emph{Successful demixing on a real 2-dimensional image from random sinusoidal features. Parameters: $n = 512 \times 512, s = 1000, m = k\times q=48000, g(x) = \sin(x)$. Image credits:~\cite{mccoy2014convexity}.}}
\label{fig:GalStar}
\end{figure}

Next, we consider the experiment illustrated in Fig.~\ref{fig:expWnoise}. In this experiment, we assume a complex sinusoid as the nonlinearity generating the feature map. In this scenario, we fix $k= 6$ (number of diagonal blocks in matrix $D$), $q=800$, and add i.i.d.\ Gaussian noise to the embeddings with increasing variance. Our goal is to compare the matched filtering part of MF-Sparse with classical spectral estimation techniques. Indeed, a natural question is whether we can use other spectral estimation approaches instead of the matched filter, and whether there is any advantage in choosing the elements of $D$ randomly. We attempt to illustrate the benefits of randomness through the experiment in Fig.~\ref{fig:expWnoise}. Here, we compare the relative error of MF-Sparse with an analogous algorithm that we call \emph{RM-Sparse}, which uses the RootMUSIC spectral estimation technique~\cite{schmidt1986multiple}. For RM-Sparse, in the line spectral estimation stage, we replace random diagonal blocks in matrix $D$ with deterministic diagonal ones, and letting the diagonal entries of the $r^{\mathrm{th}}$ block being proportional to $r$. In other words, the time samples are chosen to lie on a uniform grid, as RootMUSIC expects. As we see in Fig.~\ref{fig:expWnoise}, while RM-Sparse does recover the original $x$ for very low values of noise variance, MF-Sparse outperforms RM-Sparse when the noise level increases; this verifies the robustness of MF-Sparse to noise in the observations.

In addition, we evaluate our proposed algorithm for a real $2$D image. We begin with a $512\times 512$ test image. First, we obtain its 2D Haar wavelet decomposition and sparsify it by retaining the $s = 2000$ largest coefficients. Then, we synthesize the test image based on these largest coefficients, and multiply it by a subsampled Fourier matrix with $q = 16000$ multiplied with a diagonal matrix with random $\pm 1$ entries~\cite{krahmer2011new}. Further, we multiply this result with a block diagonal matrix $D$ with number of blocks $k=2,3$. Now, we apply the $\sin(\cdot)$ function element-wise to the resulting vector. Figure~\ref{fig:pepper} shows the recovered image using MF-Sparse. As is visually evident, with $k=2$, the quality of the reconstructed image is poor, but $k=3$ already yields a solution close to the true image. 

Finally, we present another experiment on a real $2$D image. In this case, we assume that our signal $x$ is the superposition of two constituent signals, i.e., $x = \Phi w + \Psi z$. Hence, the random feature map is given by $y = \sin(DB\left(\Phi w + \Psi z\right))$. In this setting, as illustrated in~Fig.~\ref{fig:GalStar}(a), the signal $x$ is the mixture of galaxy and star images, where the constituent coefficient vectors $w$ and $z$ are $s$-sparse signals in the DCT basis ($\Phi$) and the canonical basis ($\Psi$). In this experiment, we fix $s = 1000$ for both $w$ and $z$.  The matrices $B$ and $D$ are generated the same way as in the experiment expressed in~Fig.~\ref{fig:pepper} with $q=16000, n=2^{18}$, and $k=3$. As we can see in Figs.~\ref{fig:GalStar}(b) and~\ref{fig:GalStar}(c), running a variant of MF-Sparse along with the DHT algorithm~\cite{soltani2016fastIEEETSP17} for the second stage can successfully separate the galaxy from the stars, verifying the applicability of MF-Sparse for more generic structured inverse problems.



{{
\bibliographystyle{IEEEbib}
\bibliography{Common/mrsbiblio,Common/csbib,Common/chinbiblio,Common/kernels}
}}

\end{document}